\documentclass{ecai}
\usepackage{times}
\usepackage{graphicx}
\usepackage{latexsym}
\usepackage[inline]{enumitem}
\usepackage{amsthm}
\usepackage{caption}
\usepackage{subcaption}
\usepackage{amsmath,amssymb}
\usepackage{mathptmx}
\theoremstyle{plain}
\newtheorem{thm}{Theorem}
\newtheorem{lem}[thm]{Lemma}

\theoremstyle{definition}
\newtheorem{defn}{Definition}

\newtheorem{exmp}{Example}
\theoremstyle{remark}

\ecaisubmission

\usepackage{times}
\usepackage{microtype}
\usepackage{breakcites}
\usepackage[breaklinks=true]{hyperref}
\usepackage[mathletters]{ucs}
\usepackage{algorithm}
\usepackage{algorithmicx}
\usepackage{algpseudocode}

\title{Anytime Inference in Valuation Algebras}

\author{Abhishek Dasgupta\institute{Department of Computer Science, University of Oxford,\newline e-mail: abhishek.dasgupta@cs.ox.ac.uk}
\and Samson Abramsky\institute{Department of Computer Science, University of Oxford,\newline e-mail: samson.abramsky@cs.ox.ac.uk} }

\begin{document}

\maketitle

\begin{abstract}
The novel contribution of
this work is the construction of anytime algorithms in a generic framework, which
automatically gives us instantiations in many useful domains. We also show that semiring induced valuation algebras, an important subclass of valuation algebras are amenable to anytime inference.
Anytime inference, and inference algorithms in general have been a well-researched area in the last few decades. Inference is an important component in most pattern recognition and machine learning algorithms; it also shares theoretical connections with other branches of computer science like theorem-proving. Anytime inference is important in applications with limited space, for efficiency reasons, such as in continuous learning and robotics. In this article we
construct an anytime inference
algorithm based on principles introduced in the theory of generic
inference; and in particular, extending the work done on ordered valuation
algebras \cite{Haenni20041}.

\end{abstract}

\section{Introduction}
The inference problem is one of the most-important and well-studied problems in the field of statistics and machine learning. Inference can be considered as the (1) combination of information from various sources, which could be in the form of probability distributions from a probabilistic graphical model \cite{graphicalmodels}, belief functions in Dempster-Shafer theory \cite{dempster1968generalization,josang2012dempster} or tables in a relational database; and (2) subsequent focusing or projection to variables of interest, which corresponds to projection for variables in probabilistic graphical models, or a query in the relational database. Our work is based on the theory of generic inference \cite{generic-inference} which abstracts and generalises the inference problem across these different areas.

The utility of generic inference can be understood as an analogue to sorting, which is agnostic to the specific data type, as long as there is a total order. Generic inference generalises
inference algorithms by abstracting the essential components
of information in an algebraic structure. In \cite{LauritzenSpiegelhalter}, an algorithm was defined which solved the inference problem
on Bayesian networks, using a technique called \emph{local computation}.
It was noted in \cite{shenoy2008axioms} that the same algorithm could be
used to solve the inference problem on belief functions, and
a sufficient set of axioms were proposed for an algebraic framework that is necessary
for the generic inference algorithm. This was extended by Kohlas
into a theory of valuation algebras, and a computer implementation of
inference over valuation algebras along with concrete instantiations was
developed in \cite{pouly2010nenok}.

Generic inference as formulated in \cite{generic-inference} solves the inference problem in the exact case.
As exact inference is an \#P-hard problem \cite{valiant1979complexity}, in practice, we need frameworks for approximate inference. Approximation schemes exist for specific instances of valuation
algebras (probability potentials \cite{dechter1998bucket}, belief potentials \cite{Haenni2002103}); as well as for the generic case \cite{Haenni20041}, but there is no such generic framework for anytime inference. In this paper, we extend the approximate inference framework in \cite{Haenni20041} to support anytime inference.

In anytime algorithms, instead of an
algorithm terminating after an unspecified amount of time with a specific
accuracy, we are able to tune the accuracy via a parameter passed to the
algorithm. The algorithm can also be designed to be interruptible,
gradually improving its accuracy until terminated by the user. Such
algorithms are important in online learning where new data is being
streamed in \cite{ueno2006anytime}, in intelligent systems, decision
making under uncertainty \cite{horsch1998anytime} and robotics \cite{zilberstein1996using} where
due to the limitation of interacting in real-time there may not be
sufficient time to compute an exact solution. We shall consider \emph{interruptible} anytime algorithms which can be interrupted at any time and the approximation can be improved by resuming the
algorithm. This affords the greatest flexibility from the user's perspective, with applications of such algorithms to real-time systems such as sensor networks and path planning.

Table \ref{context-work} notes the previous work done in the area of inference algorithms, in both the generic case and for the specific case of probability potentials, and situates our work in context.

\begin{table}
\begin{tabular}{|c|c|c|}
\hline
\emph{inference}&generic&probability potentials\\
\hline
exact&\cite{generic-inference}&\cite{pearl1982reverend}\\
\hline
approximate&\cite{Haenni20041}&loopy belief propagation, \cite{dechter1998bucket}\\
\hline
anytime&[our work]&\cite{ramos2005anytime}\\
\hline
\end{tabular}
\caption{Our work, in relation to various inference algorithms and frameworks}
\label{context-work}
\end{table}
We note that the successive rows in the above table refine upon the previous one, and include it; the approximate inference framework can also
perform exact inference, and the anytime inference framework presented here gives an
approximate solution which incrementally improves with time, converging
on the solution obtained from exact inference given sufficient time.

This article is divided into the following sections. Section 2 reviews the framework of valuation algebras
and ordered valuation algebras. Section 3 introduces our extension to
ordered valuation algebras to support anytime inference, and proves soundness and completeness
theorems for anytime inference. Section 4 describes instances of the framework, including its application to anytime inference in semiring-induced valuation algebras. Section 5 gives a complexity analysis
of the algorithm. Section 6 shows implementation results of anytime inference on a Bayesian network. Section 7 concludes.

\section{Valuation Algebras}

Valuation algebras are the core algebraic structure in the theory of generic inference. In a valuation algebra, we consider the various pieces of information in an inference problem
(conditional probability distributions, belief potentials, relational database tables, etc.) as elements
in an algebraic structure with a set of axioms. We review the axioms of valuation algebra \cite{generic-inference} below, preceded by some remarks on notation.

All operations in the valuation
algebra are defined on elements denoted by lowercase
Greek letters: $\phi, \psi, \ldots$. We can think of a valuation as
the information contained by the possible values of a set of variables, which
are denoted by Roman lowercase letters (with possible subscripts): $x,y,\ldots$ and denote sets of
variables by uppercase letters: $S,T,\ldots$. Each valuation refers
to the information contained in
a set of variables which we call the \emph{domain} of a valuation, denoted
by $d(\phi)$ for a valuation $\phi$. For a finite set of variables $D$, $\Phi_D$
denotes the set of valuations $\phi$ for which $d(\phi) = D$. Thus, the set
of all possible valuations for a countable set of variables $V$ is
\begin{equation}
\Phi = \bigcup_{D\subseteq  V} \Phi_D
\end{equation}
If $\hat D = \mathcal{P}_f(V)$ the finite powerset of $V$, and $\Phi$ the set of valuations
with domains in $\hat D$; we define the following operations on $\langle \Phi, \hat D\rangle$:
\begin{enumerate*}[label=(\roman*)]
\item \emph{labeling}: $\Phi \rightarrow \hat D; \phi \mapsto d(\phi)$
\item \emph{combination}: $\Phi \times \Phi \mapsto \Phi; (\phi,\psi) \mapsto \phi \otimes \psi$
\item \emph{projection}: $\Phi \times \hat D \rightarrow \Phi; (\phi, X) \mapsto \phi^{\downarrow X}\ \mathrm{for}\ X \subseteq d(\phi)$
\end{enumerate*}

These are the basic operations of a valuation algebra. Using the view of valuations as pieces of information which refer to questions as valuations, the labelling operation
tells us which set of variables the valuation refers to; the combination operation
aggregates the information, and the projection operation focuses the information
on a particular question (query) of interest. Projection is also referred to
as \emph{focusing} or \emph{marginalization}. The following axioms are then imposed
on $\langle \Phi, \hat D \rangle$:

(A1) \emph{Commutative semigroup}: $\Phi$ is associative and commutative under $\otimes$

(A2) \emph{Labeling}: For $\phi,\psi \in \Phi$, $d(\phi \otimes \psi) = d(\phi) \cup d(\psi)$.

(A3) \emph{Projection}: For $\phi \in \Phi, X \in \hat D$ and $X \subseteq d(\phi)$,
$d(\phi^{\downarrow X}) = X$. Alternatively this is equivalent to the following
\emph{elimination} operation, $\phi^{\downarrow X} = \phi^{-(d(\phi) \backslash X)}$ where
all the variables except those in $X$ are eliminated.

(A4) \emph{Transitivity}: For $\phi \in \Phi$ and $X \subseteq Y \subseteq d(\phi)$,
$(\phi^{\downarrow Y})^{\downarrow X} = \phi^{\downarrow X}$.

(A5) \emph{Combination}: For $\phi, \psi \in \Phi$ with $d(\phi) = X,\ d(\psi) = Y$
and $Z \in D$ such that $X \subseteq Z \subseteq X \cup Y$,
$(\phi \otimes \psi)^{\downarrow Z} = \phi \otimes \psi^{\downarrow Z \cap Y}$.

(A6) \emph{Domain}: For $\phi \in \Phi$ with $d(\phi) = X$,
$\phi^{\downarrow X} = \phi$.

\parskip 0.2cm
For the intuitive reading of these axioms, we refer the reader to \cite{generic-inference,shenoy2008axioms}.

Before proceeding to approximate inference, we formally define the inference problem:
\begin{defn}
The \emph{inference problem} is the task of computing
\begin{equation}
\phi^{\downarrow X} = (\phi_1 \otimes \cdots \otimes \phi_r)^{\downarrow X}
\end{equation}
for a given knowledgebase $\{\phi_1,\ldots,\phi_r\} \subseteq \Phi$; domain $X$ is the \emph{query} for the inference problem.
\end{defn}

Next we consider \emph{approximate inference}. Existing approximation schemes, like the mini-bucket scheme \cite{dechter1998bucket} are either not general enough or do not provide a reliable measure of the approximation and how to improve the approximation in an anytime algorithm. In this article, we have used the ordered valuation algebra framework defined in \cite{Haenni20041} as a basis for constructing an anytime algorithm. We thus review the extra axioms of the ordered valuation algebra framework, which
introduces the notion of a partial order into the valuation algebra, and defines a partial combination operator
$\otimes_t$ to construct approximate inference algorithms.

Firstly we define a relation $\succeq$ which represents an information ordering.
If $\phi,\phi'$ are
two valuations, then $\phi \succeq \phi'$ means that $\phi$ is
\emph{more complete} than $\phi'$. Intuitively, the information contained
in $\phi$ is more informative and a better approximation than the information
contained by $\phi'$; generally this means $\phi'$ has a more compact or sparse representation
than $\phi$. Furthermore, we assume that this relation is a \emph{partial order}. It is also reasonable to assume that approximations are only valid for valuations
with equal domains; thus $\phi \succeq \phi'$ implies $d(\phi) = d(\phi')$ for all $\phi,\phi' \in \Phi$.
Thus $\succeq$ actually defines separate completeness relations $\succeq_D$ for each sub-semigroup $\Phi_D$.

We also impose the condition of each sub-semigroup $\Phi_D$ having a \emph{zero}
element, denoted by $n_D$, where $\phi \otimes n_D = n_D \otimes \phi = n_D$. For notational simplicity we
shall also denote the neutral element by $\varnothing$ (without a subscript), denoting the appropriate
neutral element corresponding to a particular domain.

An ordered valuation algebra is still a valuation algebra, so it retains
all the axioms (A1)-(A6) introduced previously. The additional axioms are about how $\succeq$ behaves under
combination and marginalization:

(A7) \emph{Partial order}: There is a partial order $\succeq$ on $\Phi$
such that $\phi \succeq \phi'$ implies $d(\phi) = d(\phi')$ for all
$\phi,\phi' \in \Phi$.

(A8) \emph{Zero element}: We assume that the zero element for the combination operation, $n_D$ is the least element of the approximation order $\succeq_D$ for all $D \subseteq V$. Also, since
zero elements for a particular domain are unique, $n_{D_1} \otimes n_{D_2} = n_{D_1 \cup D_2}$ for
$D_1,D_2 \subseteq V$. Also, $n_D^{\downarrow D'}= n_{D'}$ for all $D' \subseteq D$.

(A9) \emph{Combination preserves partial order}: If $\phi_1,\phi_1',\phi_2,\phi_2' \in \Phi$
are valuations such that $\phi_1 \succeq \phi_1'$ and $\phi_2 \succeq \phi_2'$, then
$\phi_1 \otimes \phi_2 \succeq \phi_1' \otimes \phi_2'$

(A10) \emph{Marginalisation preserves partial order}: If $\phi,\phi' \in \Phi$ are
valuations such that $\phi \succeq \phi'$, then $\phi^{\downarrow D} \succeq \phi'^{\downarrow D}$
for all $D \subseteq d(\phi) = d(\phi')$.

\begin{defn}
The \emph{time-bounded combination operator} \cite{Haenni20041} $\otimes_t: \Phi \times \Phi \rightarrow \Phi$ is used to approximate the exact computation during the propagation phase. $\otimes_t$ performs a partial combination of two valuations within time $t$ units, where $t \in \mathbb{R}^+$. The following properties are satisfied by $\otimes_t$:

(R1) $\phi_1 \otimes \phi_2 \succeq \phi_1 \otimes_t \phi_2$.

(R2) $\phi_1 \otimes_{t'} \phi_2 \succeq \phi_1 \otimes_t \phi_2$ for all $t' > t$.

(R3) $\phi_1 \otimes_0 \phi_2  = n_{d(\phi_1) \cup d(\phi_2)}$.

(R4) $\phi_1 \otimes_\infty \phi_2 = \phi_1 \otimes \phi_2$.
\end{defn}

\begin{defn}
    \label{def-ordered-valuation-algebra}
Such a system $\langle \Phi,V,\succeq,d,\otimes,\downarrow,\otimes_t \rangle$ of valuations $\Phi$, variables $V$, a completeness relation $\succeq$ and a time-bounded combination operation $\otimes_t$ is called an \emph{ordered valuation algebra},
if the labeling operations $d$, combination $\otimes$ and marginalization $\downarrow$
satisfy (A1)-(A10).
\end{defn}

\begin{defn}
A binary join tree (BJT) $N = \langle V,E\rangle$ corresponding to a
knowledgebase $\{\phi_1,\ldots,\phi_r\}$ is a covering junction tree for
the inference problem, constructed in a manner such that the tree is binary. The valuations in the knowledgebase
form the leaves of the tree, thus $|V(N)| = 2r - 1, |E(N)| = 2r - 2$, while
the query $X \subseteq d(\mathit{root}(N))$. Inference takes place by message passing in the BJT (for details of the algorithm, see \cite{shenoy1997binary, Haenni20041}).
In the next section we shall modify this message passing algorithm to cache partial
valuations for anytime inference.
\end{defn}

\section{Anytime Ordered Valuation Algebras}

In this section, we augment ordered valuation algebras in a structure
we refer to as \emph{anytime ordered valuation algebras}. We
introduce the extension, and in the following section give examples of
anytime ordered valuation algebras. The primary purpose of introducing anytime ordered
valuation algebras is to develop an anytime inference algorithm
within the framework of generic inference. Such extensions
preserve the generic structure of valuation algebras, but add restrictions to simplify or add features to
the inference algorithm; in another instance, valuation algebras were extended to weighted valuation algebras
to study communication complexity \cite{Pouly05minimizingcommunication}.

Before defining anytime ordered valuation algebras, we define
a couple of prerequisites; the \emph{composition operation} and a \emph{truncation function}.

\begin{defn}
The \emph{composition operator}, $\oplus:\Phi \times \Phi \rightarrow \Phi; (\phi', \phi'') \mapsto \phi$ combines valuations $\phi'$ and $\phi''$ into a valuation $\phi$ more complete than either ($\phi \succeq \phi', \phi \succeq \phi''$). This is not to be confused with the combination operation $\otimes$ which generally combines valuations from different domains. The valuations being composed belong to the same
approximation order $\succeq_D$, where $D = d(\phi') = d(\phi'') = d(\phi)$. It is natural in this context to consider whether composition should be a supremum operation. However, this cannot be assumed in general.
\end{defn}
\begin{defn}
The \emph{truncation function} $\rho: \Phi \times \mathbb{R}^+ \rightarrow \Phi$
performs a truncation of the information contained in the valuation, according
to the real valued parameter. Also, $\rho$ is defined to be \emph{monotonically increasing}
with the real valued parameter, thus $\rho(\phi,k)\succeq \rho(\phi,k')\ \mathrm{whenever}\ k \ge k'$.
\end{defn}
The time-bounded combination operation $\otimes_t$ can be recast such that truncation of the original pair of valuations followed by exact combination is equivalent to doing a time-bounded combination:
\begin{equation}
\phi_1 \otimes_t \phi_2 = \rho(\phi_1, k_1) \otimes \rho(\phi_2, k_2)
\label{otimes-t-k1-k2}
\end{equation}
The parameters $k_1, k_2$ determining the truncated portions of $\phi_1, \phi_2$ will be important
later in defining the partial valuations which will be used in the refinement algorithm for
anytime inference. 
As $k_1, k_2$ are parameters that depend on the particular valuations $\phi_1, \phi_2$ and the time $t$, this assumes a function $K(\phi_1,\phi_2,t) = (k_1, k_2)$.

Following these two definitions, we extend the system of axioms (A1-A10) for
ordered valuation algebras, with the properties (P1) and (P2):

(P1) The combination operation $\otimes$ distributes over $\oplus$:
\begin{flalign}
\quad& (\phi_1' \oplus \phi_1'')\otimes(\phi_2' \oplus \phi_2'')=\nonumber &\\
\quad& (\phi_1' \otimes \phi_2') \oplus \underbrace{(\phi_1' \otimes \phi_2'')
 \oplus (\phi_1'' \otimes \phi_2') \oplus (\phi_1'' \otimes \phi_2'')}_{\Large \textsc{refine}'(\phi_1',\phi_1'',\phi_2',\phi_2'')} &
\end{flalign}

Here, $\phi_1' \otimes \phi_2' = \rho(\phi_1, k_1) \otimes \rho(\phi_2, k_2)$ is a truncated valuation of the exact combined valuation $\phi_1 \otimes \phi_2$; $\textsc{refine}'$ is the part of the exact valuation that needs to be composed with the truncated valuation $\phi'_1 \otimes \phi'_2$ to complete the valuation. We also use the time-bounded operation $\textsc{refine}'_t$ for the same operation bounded by a time $t$, with an analogous definition in terms of truncation functions as $\otimes_t$ in equation \ref{otimes-t-k1-k2}:
\begin{equation}
\textsc{refine}^\prime_t(\phi_1',\phi_1'',\phi_2', \phi_2'') =
\textsc{refine}^\prime(\phi_1', \rho(\phi_1'', k_1), \phi_2', \rho(\phi_2'', k_2))
\end{equation}
where the parameters $k_1, k_2$ are obtained from an assumed function $K'(\phi_1,\phi_1',\phi_2',\phi_2'', t) = (k_1,k_2)$.

\vskip 0.2cm
(P2) The projection operation $\downarrow$ distributes over $\oplus$:
\begin{equation}
(\phi' \oplus \phi'')^{\downarrow D} = \phi'^{\downarrow D} \oplus
 \phi''^{\downarrow D}, D \subseteq d(\phi).
\end{equation}
\vskip 0.2cm
We can now formally define the anytime ordered valuation algebra.

\begin{defn}
An \emph{anytime ordered valuation algebra} is an ordered valuation
algebra $\langle V,\Phi,d,\otimes,\downarrow,\otimes_t,\succeq\rangle$ with the additional operations of composition
$\oplus$ and the function $\rho$, making the structure $\langle V,
\Phi,d,\rho,\otimes,\downarrow,\oplus,\otimes_t,\succeq\rangle$,
which satisfies properties (P1) and (P2).
\end{defn}

We show by construction that
the composition operator\\ $\oplus: \Phi \times \Phi \rightarrow \Phi$ with (P1, P2) along with the truncation function $\rho: \Phi \times \mathbb{R}^+
\rightarrow \Phi$ is \emph{sufficient} to construct a
refinement algorithm to improve the accuracy of a valuation.

To describe a refinement algorithm to improve upon the result provided by
\textsc{inward}$(N,t)$, we need to cache the partial valuations at each
step so that we can use $\textsc{refine}'$ to improve upon them. We use
a modified version of the propagation algorithm \cite{shenoy1997binary,Haenni20041}, where $\tau$ and
$\bar\tau$ store the partial and complementary partial valuations
respectively for a particular BJT node, where the complementary partial valuation $\bar\rho(\phi, k)$ is such that $\bar\rho(\phi, k) \oplus \rho(\phi, k) = \phi$. In the following procedures, $\Delta(n) = d(n)\backslash d(P(n))$ is the set of variables to be eliminated as we propagate messages to the parent node. To get the solution to the inference problem at the final step, we also define $\Delta(\mathit{root}(N)) = d(\mathit{root}(N)) \backslash X$ where $X$ is the query. There are $r$ valuations in the knowledgebase resulting in $r-1$ combination steps in the BJT. $P(n)$ is the parent of $n$, $\phi(n)$ is the valuation at node $n$, $\phi_s(n)$ is the message from $n$ to $P(n)$; $L(n), R(n)$ are the left and right nodes of $n$ respectively and
\begin{equation}
\mathit{next}(N) = \{n \in N: \phi_s(n) = \mathit{nil},
  \phi_s(L(n)) \ne \mathit{nil}, \phi_s(R(n)) \ne \mathit{nil} \}
\end{equation}
Both $\textsc{inward}(N,t)$ and $\textsc{refine}(N,t)$ return valuations which are the (approximate) solution to the inference problem.

\hrulefill
\begin{algorithmic}[1]
\Procedure{inward}{$N,t$}
\State $s \leftarrow r-1$;
\State initialise timer to $t$ units.
\State for all $n \in \mathit{leaves}(N)$ do $\phi_s(n) \leftarrow \phi(n)^{-\Delta(n)}$
\While {$\mathit{next}(N) \ne \emptyset$}
\State \textbf{select} $n \in next(N)$
\State $(k_1, k_2) \gets K(\phi_s(L(n)), \phi_s(R(n)), t/s)$
\State $\phi(n) \gets \phi_s(L(n)) \otimes_{t/s} \phi_s(R(n))$
\State $\tau(L(n)) \gets \rho(\phi_s(L(n)), k_1)$
\State $\tau(R(n)) \gets \rho(\phi_s(R(n)), k_2)$
\State $\bar\tau(L(n)) \gets \bar\rho(\phi_s(L(n)), k_1)$
\State $\bar\tau(R(n)) \gets \bar\rho(\phi_s(R(n)), k_2)$
\State $\phi_s(n) \leftarrow \phi(n)^{-\Delta(n)}$
\State $s \leftarrow s-1$
\State $t \leftarrow \mathit{timer}()$
\EndWhile

\State \Return $\phi_s(\mathit{root}(N))$
\EndProcedure
\end{algorithmic}
\hrulefill

We can use the cached partial valuations in $\tau$ and $\bar\tau$ to define
the refinement algorithm that follows in a similar manner to the algorithm in \cite{Haenni2002103}.

\hrulefill
\begin{algorithmic}[1]
    \Procedure{refine}{$N,t$}
\State \textbf{initialise} timer to $t$ units
\State $s\gets r-1$
\While{$next(N)\not=\emptyset$}
\State \textbf{select} $n \in next(N)$
\State $(k_1, k_2) \gets K'(\tau(L(n)), \bar\tau(L(n)), \tau(R(n)), \bar\tau(R(n)), t/s)$
\State $\nu \gets \textsc{refine}'_{t/s}(\tau(L(n)), \bar\tau(L(n)), \tau(R(n)), \bar\tau(R(n)))$
\State $t \gets timer()$
\State $\tau(L(n)) \gets \tau(L(n)) \oplus \rho(\bar\tau(L(n)), k_1)$
\State $\tau(R(n)) \gets \tau(R(n)) \oplus \rho(\bar\tau(R(n)), k_2)$
\State $\bar\tau(L(n)) \gets \bar\rho(\bar\tau(L(n)), k_1)$
\State $\bar\tau(R(n)) \gets \bar\rho(\bar\tau(R(n)), k_2)$
\State $\phi(n)\gets \phi(n) \oplus \nu$
\State $\bar\tau(n)\gets\bar\tau(n) \oplus \nu^{-\Delta(n)}$
\State $s\gets s-1$
\EndWhile\label{euclidendwhile}
\State \textbf{return} $\phi_s(root(N))$
\EndProcedure
\end{algorithmic}
\hrulefill	

This procedure refines the existing valuations in the binary join tree $N$, taking at most time $t$ units. We ensure that the algorithm
is interruptible in lines 9--12 using appropriate caching of partial valuations. A diagram of the truncation of a valuation is shown below
to illustrate anytime refinement.
\begin{figure}[htp]
\centering
\includegraphics[scale=0.7]{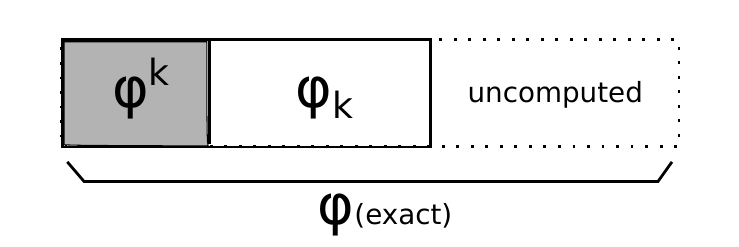}
\end{figure}

Here, and in the following proof, the notation $\phi^k := \rho(\phi,k)$
and $\phi_k := \bar\rho(\phi,k)$. We shall also abbreviate the notation $\tau(L(n))$ as $\tau_L$ and $\bar\tau(L(n))$ as $\bar\tau_L$ (accordingly for $R(n)$), and $\bar\tau(n)$ as $\bar\tau$. The shaded region $\phi^k$ is the
part that has already been combined, while $\phi_k$ represents the
cached part that has not been combined yet. The dotted
region represents the part of $\phi$ that is yet uncomputed, due to
truncated messages from child nodes; line 14 in $\textsc{refine}(N,t)$ shrinks the uncomputed portion by extending $\bar\tau$.

\begin{thm}[Soundness of anytime inference]
If $\phi_{[t_0,t_1,\ldots,t_j]}$ is the valuation returned after the following invocations:
$\left[\textsc{inward}(N_0,t_0>0),\ \textsc{refine}(N_1, t_1), \ldots,\ \textsc{refine}(N_j, t_j) \right]$, where $N_{k+1}$ is the modified BJT with the cached valuations after step $k$, then
$\phi_{[t_0]} \preceq \phi_{[t_0,t_1]} \preceq \cdots \preceq \phi_{[t_0,t_1,\ldots,t_j]} \preceq \cdots \preceq \phi$ where $\phi$ is the exact valuation. The sequence becomes strictly increasing (upto the exact valuation) if $t_i > t_\epsilon$ for all $i>0$ where $t_\epsilon$ is the minimum time required for the refinement to update one valuation.
\end{thm}
\begin{proof}
We split the proof into two parts: (S1) proving that the sequence of valuations
returned from successive calls to $\textsc{refine}$ are partially ordered and (S2) showing the upper bound is the exact valuation, to which the partial valuations
converge after a finite time.

Proving (S1) is trivial; for each node, $\phi$ is updated once (line 13), thus $\phi' = (\phi \oplus \nu) \succeq \phi$, where $\phi'$ is the valuation at node $n$ after a call to $\textsc{refine}$. Using transitivity of the partial order, we obtain (S1). In the case
when $t_i > t_\epsilon$, at least one valuation is updated, resulting in $\nu \succ \varnothing$, which gives $\phi' \succ \phi$.

To prove (2) we shall note the following statements

(T1) $(\phi_k)^m = (\phi^{k+m})_k$

(T2) $\phi^k \oplus \phi_k = \phi$

(T3) $\phi^k \oplus (\phi_k)^m = \phi^{k+m}$

(T4) $(\phi_k)_m = \phi_{k+m}$

For notational simplicity, only for the following proof, we denote $\phi\psi := \phi\otimes\psi$ and $+ := \oplus$.

Since each node is only updated once, we can consider a particular node; let's denote by $\phi$ the valuation at node $n$ after $\textsc{inward}(N,t_0)$. If
$(k_1, k_2)$ are the parameters obtained from $K'$ in $\textsc{refine}(N_1,t_1)$ then the updated valuation
$\phi' = \phi + \tau_L\bar\tau_R^{k_2} + \bar\tau_L^{k_1}\tau_R +
\bar\tau_L^{k_1}\bar\tau_R^{k_2}$, where $\phi = \tau_L\tau_R$.

Here we note that we can replace $(\bar\tau_{L,R})^k$ with their exact counterpart $(\bar\tau^{\infty}_{L,R})^k$, where
we use the $\bar\tau^\infty$ to denote the exact valuation. This can be done as the truncation function is invariant
under extension of the valuation to incorporate previously uncomputed information. Following this, we shall
drop the superscript and use $\bar\tau_L$ to denote $\bar\tau^\infty_L$.

Then if we consider a subsequent call, $\textsc{refine}(N_2,t_2)$,
$\phi'' = \phi' + \tau_L'\bar\tau_R^{\prime m_2} + \bar\tau_L^{\prime m_1}\tau_R' +
\bar\tau_L^{\prime m_1}\bar\tau_R^{\prime m_2}$.
where the additional prime indicates the the value for this iteration, and
$(m_1,m_2)$ are the parameters obtained from $K'$.

From lines 9--12 in \textsc{refine} we get: $\tau_L' = \tau_L + \bar\tau_L^{k_1}$,
$\tau_R' = \tau_R + \bar\tau_R^{k_2}$,
$\bar\tau'_L = (\bar\tau_L)_{k_1}$,
$\bar\tau'_R = (\bar\tau_R)_{k_2}$

Expanding $\phi''$ we get:
\begin{eqnarray*}
\phi''&=&\tau_L\tau_R + \tau_L\bar\tau_R^{k_2} + \bar\tau_L^{k_1}\tau_R +
\bar\tau_L^{k_1}\bar\tau_R^{k_2}\\
&&+\ (\tau_L + \bar\tau_L^{k_1})(\bar\tau_{R,k_2})^{m_2} +
(\bar\tau_{L,k_1})^{m_1}(\tau_R + \bar\tau_R^{k_2}) + (\bar\tau_{L,k_1})^{m_1}(\bar\tau_{R,k_2})^{m_2}\\
&=&\tau_L\tau_R + \tau_L\bar\tau_R^{k_2} + \bar\tau_L^{k_1}\tau_R +
\bar\tau_L^{k_1}\bar\tau_R^{k_2} + \tau_L(\bar\tau_{R,k_2})^{m_2}\\
&&+ (\bar\tau_L^{k_1})(\bar\tau_{R,k_2})^{m_2} +
(\bar\tau_{L,k_1})^{m_1}\tau_R + (\bar\tau_{L,k_1})\bar\tau_R^{k_2} +
(\bar\tau_{L,k_1})^{m_1}(\bar\tau_{R,k_2})^{m_2}\\
&=&\tau_L\tau_R + \tau_L\bar\tau_R^{k_2+m_2} +
\bar\tau_L^{k_1+m_1}\tau_R + \bar\tau_L^{k_1+m_1}\bar\tau_R^{k_2+m_2}
\end{eqnarray*}
Here we use (T1,T3) to simplify the expression. Note that
this is the same form as $\phi' = \phi + \tau_L\bar\tau_R^{k_2} + \bar\tau_L^{k_1}\tau_R + \bar\tau_L^{k_1}\bar\tau_R^{k_2}$, with $k_1 \rightarrow k_1+m_1,\ k_2 \rightarrow k_2 + m_2$. Thus, subsequent calls to $\textsc{refine}$ will always
result in $\phi$ having the same form by induction. From the definition of the
truncation function, $\phi^k \succeq \phi^{k'}$ for $k \ge k'$, from which (S1)
follows as well, by preservation of partial order under combination and composition.
To show (S2) we note that for finite valuations, there exists $k$, such that $\phi^k = \phi$. As the exponent is monotonically increasing with subsequent calls
to $\textsc{refine}$, we shall eventually get $\phi_{[t_0,t_1,\ldots,t_j]} = \tau_L\tau_R
+ \tau_L\bar\tau_R + \bar\tau_L\tau_R + \bar\tau_L\bar\tau_R =
(\tau_L+\bar\tau_L)(\tau_R+\bar\tau_R)$, the exact valuation at node $n$. Thus, we
shall eventually get the exact valuation at the root after finite invocations of $\textsc{refine}$.

\end{proof}

\begin{thm}[Completeness of anytime inference]
If $\phi_{[t_0,t]}$ is the valuation returned after the following invocations: $\left[\textsc{inward}(N,t_0>0),\ \textsc{refine}(N', t)\right]$, where $N'$ is the modified BJT with the cached valuations after the call to $\textsc{inward}(N,t_0)$, then
there exists a $T$ such that for all $t \ge T$
$\phi_{[t_0,t]} = \phi = (\bigotimes_{\psi \in \Phi} \psi)^{\downarrow X}$, the
exact solution to the inference problem.
\end{thm}
\begin{proof}
We consider two cases:

\textbf{Case 1}: $\textsc{inward}(N,t_0)$ has performed exact inference.

We shall show that $\textsc{refine}(N,t)$ is a null operation which
does not change $\phi,\tau,\bar\tau$; then the statement of the theorem follows if we set $T = t_0$.

$\phi' = \phi \oplus \nu$ (line 13), so if we show $\nu = \varnothing$, we are done.

$\nu = \textsc{refine}'_{t/s}(\tau_L,\bar\tau_L,\tau_R,\bar\tau_R)$, but $\bar\tau_L = \bar\tau_R = \varnothing$ as $\bar\tau$ represents the partial valuation that has not been combined, which is null for the exact inference case. Thus $\nu = \varnothing$.

\textbf{Case 2}: $\textsc{inward}(N,t_0)$ gives a partial result.

In general, $\nu$ is also a partial valuation due to the time restriction. 
Since we are operating on finite datasets,
the combination operation at a particular node in $\textsc{refine}'$ takes a finite amount of time, say $t_n$. Thus $\textsc{refine}'_{t_n}$ at a node $n$ is the
exact refinement, making $\phi(n)$ exact after line 13, and thus $m(root(N))$ is exact after completion of the propagation. So we set $T = \sum_{n \in V} t_n$ to
get the time bound, such that for all $t \ge T$ we get the exact result.
\end{proof}

\section{Instances of anytime ordered valuation algebras}

In the following sections, we describe instances of anytime ordered valuation algebras. Specifically we show that the important class of semiring induced valuation algebras,
\cite{kohlas2008semiring}, can be considered as anytime ordered valuation algebras. We also remark on the application of our framework to belief potentials.

\subsection{Semiring induced valuation algebras}

Semiring induced valuation algebras are a subclass of valuation algebras with several useful instances like probability potentials and disjunctive normal forms. We use the definition of semiring induced valuation algebras from \cite{kohlas2008semiring} and review the following standard notation. The semiring is denoted by $\mathcal{A} = \langle A,+,\times\rangle$ with the semiring operations $+,\times$ on a set $A$, where $+, \times$ are assumed to be commutative and associative, with $\times$ distributing over $+$. Lowercase letters like $x$
are variables, with a corresponding finite set of values for $x$, called the \emph{frame} of $x$
and denoted by $\Omega_x$. Each $\Omega_x$ also has an associated total order on its elements. If the frame has two elements, then it is the frame of a \emph{binary
variable}. If the binary elements represent true and false, then we call the variable \emph{propositional}.
For a domain $D \subseteq V$ where $V$ is the set of all variables in the system,
the corresponding set of possible values becomes the Cartesian product $\Omega_D = \prod \{\Omega_x: x \in D\}$, whose
elements $\mathbf{x} \in \Omega_D$ are called \emph{D-configurations} or \emph{D-tuples}. For a subset $D' \subseteq D$,
$\mathbf{x}^{\downarrow D'} \in \Omega_{D'}$ is the projection of $\mathbf{x}$ to $D'$. Where $D$ is empty,
we use the convention that the frame is a singleton set: $\Omega_\phi = \{\diamond\}$.
Any set of $D$-configurations can be ordered using a lexicographical order.

\begin{defn}
An $\mathcal{A}$-valuation $\phi$ with domain $D$ associates a value in $A$ with each configuration $\mathbf{x} \in \Omega_D$, i.e. $\phi$ is a function $\phi: \Omega_D \rightarrow A$. 
\end{defn}

The set of all such $\mathcal{A}$-valuations with a domain $D$ is denoted by $\Phi_D$, and the union of all such sets with $D \subseteq V$ is the set of all $\mathcal{A}$-valuations $\Phi$. The operations $+,\times$ on $A$ then induce a valuation algebra structure on $\langle \Phi, \mathcal{P}_f(V) \rangle$ where $\mathcal{P}_f(V)$ is the finite powerset of the set of variables $V$
\cite[Theorem 2]{kohlas2008semiring}, using the following definitions of combination and projection:

\begin{enumerate}
\item \emph{Combination}: $\otimes: \Phi \times \Phi \rightarrow \Phi$ defined for $\mathbf{x} \in \Omega_{d(\phi) \cup d(\psi)}$ by
\begin{equation}
\phi \otimes \psi (\mathbf{x}) = \phi(\mathbf{x}^{\downarrow d(\phi)}) \times \psi(\mathbf{x}^{\downarrow d(\psi)})
\end{equation}

\item \emph{Projection}: $\downarrow: \Phi \times D \rightarrow \Phi$ defined for all $\phi \in \Phi$ and $T \subseteq d(\phi)$ for $\mathbf{x} \in \Omega_T$ by
\begin{equation}
\phi^{\downarrow T}(\mathbf{x}) = \sum_{\mathbf{z} \in \Omega_{d(\phi)}:\ \mathbf{z}^{\downarrow T} = \mathbf{x}} \phi(\mathbf{z})
\end{equation}
\end{enumerate}

\begin{thm}
\label{semiring-ova}
Semiring induced valuation algebras, provided the underlying semiring has a zero element, form an ordered valuation algebra.
\end{thm}
\begin{proof}
To show semiring induced valuation algebras are an ordered valuation algebra, we have to show
(A7-A10):

(A7) The \emph{preorder} $\succeq$ is defined by $\phi \succeq \phi'$ iff $\phi(\mathbf{x}) \succeq_A \phi'(\mathbf{x})$ for all $\mathbf{x} \in \Omega_{d(\phi)}$, where $\succeq_A$ is the preorder on the semiring \cite[Prop. 1, p1362]{kohlas2008semiring} defined as $b \succeq_A a$ iff $a = b$ or there exists $c$ such that $a+c=b$, with $d(\phi) = d(\phi')$ as it only makes sense to compare valuations on the same domain. However we need a \emph{partial order} for this axiom, which is possible if the additive monoid is positive, has a zero element and is cancellative:
\begin{lem}
The preorder $\preceq$ defined on a positive, cancellative, commutative monoid, $\langle A, + \rangle$ with a zero element, is a partial order.
\end{lem}
\begin{proof}
A preorder implies $a \preceq b$ iff $a + c = b$. For a partial order, we need asymmetry: if $a \preceq b$ and $b \preceq a$, then $a=b$.

$a \preceq b$ implies there exists $c$ such that $a + c = b$; similarly there exists $d$ such that $b + d = a$; substituting gives us $b + d + c = b + 0$,
the cancellative property implies $d + c = 0$ and the positivity property
implies $c = d = 0$, implying $a = b$, and we have a partial order.
\end{proof}

(A8) \emph{Zero element}: Most common
instances of semiring induced valuation algebras have a zero element. Specifically semirings with \emph{zero elements} induce valuation algebras with
the zero element $n_D$ such that $n_D(\mathbf{x}) = 0$ for all $\mathbf{x} \in \Omega_D$.

(A9, A10) \emph{Combination and marginalisation preserve partial order}. This follows from the fact that $\times$ and $+$
preserve partial order in the underlying semiring structure.

\end{proof}

Having shown that semiring induced valuation algebras satisfy the ordered valuation algebra axioms (A7--A10) provided the underlying semiring has a zero element and the additive commutative monoid is cancellative and positive, we proceed to define the composition and truncation functions
for semiring induced valuation algebras.

\begin{enumerate}
\item We denote the composition operator on semiring induced valuation algebras as
    $(\phi \oplus \phi') (\mathbf{x}) = \phi (\mathbf{x}) + \phi' (\mathbf{x}),\ d(\phi) = d(\phi')
    $
\item The function $\rho$ is defined on the semiring induced valuation algebra as
$\rho(\phi, k)=$ the first $k$ (lexicographically ordered on $\mathbf{x}$) elements of $\mathrm{graph}(\phi)$; where
$\mathrm{graph}(\phi) = \{(\mathbf{x},\phi(\mathbf{x}))\ \vert\ \mathbf{x} \in \Omega_{d(\phi)} \}$. For efficient implementation, we only store $(\mathbf{x},\phi(\mathbf{x}))$ where $\phi(\mathbf{x}) \ne 0$.

In case the semiring has a total order (as in the case of probability potentials), we order
the configurations in decreasing weight order: $[ (\mathbf{x}_i, \phi(\mathbf{x}_i)),
\ldots ]$ where $\phi(\mathbf{x}_i) \ge \phi(\mathbf{x}_j)$ for $i \le j$.
\end{enumerate}

We also define the time-bounded combination operation $\phi_1 \otimes_t \phi_2$,
where $L_{\phi_1} = \left[(\mathbf{x_1}, \phi_1(\mathbf{x_1})), \ldots \right]$,
and $L_{\phi_2} = \left[(\mathbf{y_1}, \phi_2(\mathbf{y_1})), \ldots \right]$. $\mathbf{xy}$ denotes the configuration in $\Omega_{d(\phi_1) \cup d(\phi_2)}$ such that
$(\mathbf{xy})^{\downarrow d(\phi_1)} = \mathbf{x}$ and
$(\mathbf{xy})^{\downarrow d(\phi_2)} = \mathbf{y}$.

We define helper functions $\textsc{insert}$, which inserts a combination into
the configuration space provided there is a common support and $\textsc{combine-extend}$ which incrementally adds combinations into the configuration and updates
the state, going from the state $\rho(\phi_1, i) \otimes \rho(\phi_2, j)$
to $\rho(\phi_1, i + i') \otimes \rho(\phi_2, j + j')$. Finally
we define $\textsc{combine}$ which performs the combination
operation within the allocated time constraint.

\hrulefill

\begin{algorithmic}[1]
\Function{insert}{$\phi_1, \phi_2, i, j, L$}
\State $\mathbf{x} = L_{\phi_1}; \mathbf{y} = L_{\phi_2}$
\If{$\mathbf{x}_i^{\downarrow D_1 \cap D_2} = \mathbf{y}_r^{\downarrow D_1 \cap D_2}$}
\State insert $[\mathbf{x}_i\mathbf{y}_j, \phi_1(\mathbf{x}_i) \times \phi_2(\mathbf{y}_j)]$ into $L$.
\EndIf
\EndFunction
\end{algorithmic}
\hrulefill

\begin{algorithmic}[1]
\Function{combine-extend}{$\phi_1, \phi_2, \langle i,j,L \rangle, i', j'$}
\For{$k \gets 1$ to $i+i'$}
\For{$m \gets j$ to $j+j'$}
\State $\textsc{insert}(\phi_1, \phi_2, k,m,L)$
\EndFor
\EndFor
\For{$k \gets i$ to $i+i'$}
\For{$m \gets 1$ to $j+j'$}
\State $\textsc{insert}(\phi_1, \phi_2, k,m,L)$
\EndFor
\EndFor
\State \Return $\langle i,j,L \rangle$
\EndFunction
\end{algorithmic}

\hrulefill
\begin{algorithmic}[1]
\Function{combine}{$\phi_1, \phi_2,t$}
\State $L \leftarrow \langle \rangle; i \leftarrow 1; j \leftarrow 1;
n_1 \gets |L_{\phi_1}|; n_2 \gets |L_{\phi_2}|$
\State initialise timer to $t$ units
\While{$\mathit{timer}() > 0$ and $i \le n_1$ and $j \le n_2$}
\State $\langle i, j, L \rangle \gets \textsc{combine-extend}(\phi_1, \phi_2, \langle i,j,L \rangle,0,1) $
\If{not $\mathit{timer}() > 0$}
\State \textbf{break}
\EndIf
\State $\langle i, j, L \rangle \gets \textsc{combine-extend}(\phi_1, \phi_2, \langle i,j,L \rangle,1,0) $
\EndWhile
\If{$i > n_1$}
\State $m \gets j+1$
\While{$\mathit{timer}() > 0$ and $m \le n_2$}
\State $\langle i,j,L \rangle \gets \textsc{combine-extend}(\phi_1,\phi_2,\langle i,j,L \rangle, 0, 1)$
\State $m \gets m + 1$
\EndWhile
\Else
\State $m \gets i+1$
\While{$\mathit{timer}() > 0$ and $m \le n_1$}
\State $\langle i,j,L \rangle \gets \textsc{combine-extend}(\phi_1,\phi_2,\langle i,j,L \rangle, 1,0)$
\State $m \gets m + 1$
\EndWhile

\EndIf
\State \Return valuation corresponding to $L$
\EndFunction
\end{algorithmic}
\hrulefill

\begin{thm}
Semiring induced valuation algebras, provided the underlying semiring has a zero element, along with the composition operator and the truncation function
defined above form an anytime ordered valuation algebra.
\end{thm}

\begin{proof}
Semiring induced valuation algebras form an ordered valuation algebra as shown in Theorem~\ref{semiring-ova}. To show that they also constitute an anytime ordered valuation algebra, we have to show properties (P1, P2), i.e. combination and projection distribute over $\oplus$:

(P1) If $p_1 = p_1' \oplus p_1''$ and $p_2 = p_2' \oplus p_2''$ then we have to show that: $
p_1 \otimes p_2 = (p_1' \otimes p_2') \oplus (p_1' \otimes p_2'') \oplus (p_1'' \otimes p_2') \oplus (p_1'' \otimes p_2'')$.

LHS applied to $\mathbf{x}$ is $p_1(\mathbf{x}^{\downarrow S}) \times p_2(\mathbf{x}^{\downarrow
T})$, where $d(p_1) = S$ and $d(p_2) = T$.
\begin{eqnarray*}
\text{RHS is }(p_1'(\mathbf{x}^{\downarrow S}) \times p_2'(\mathbf{x}^{\downarrow T})) +
(p_1'(\mathbf{x}^{\downarrow S}) \times p_2''(\mathbf{x}^{\downarrow T})) +\\
(p_1''(\mathbf{x}^{\downarrow S}) \times p_2'(\mathbf{x}^{\downarrow T})) +
(p_1''(\mathbf{x}^{\downarrow S}) \times p_2''(\mathbf{x}^{\downarrow T}))\\
= (p_1'(\mathbf{x}^{\downarrow S}) + p_1''(\mathbf{x}^{\downarrow S})) \times
(p_2'(\mathbf{x}^{\downarrow S}) + p_2''(\mathbf{x}^{\downarrow T}) = \text{LHS}\\
\text{using distributivity of} \times \text{over}\ +.
\end{eqnarray*}

(P2) We have to show that if $p = p' \oplus p''$ that $p^{\downarrow D} =
p'^{\downarrow D} \oplus p''^{\downarrow D}$, where $D \subseteq d(p)$. The LHS applied to $\mathbf{x}$ is $p^{\downarrow D}(\mathbf{x}) = \sum_{\mathbf{z}^{\downarrow D} = \mathbf{x}} p(\mathbf{z}) = \sum_{\mathbf{z}^{\downarrow D} = \mathbf{x}} (p' \oplus p'')(\mathbf{z})$, and the RHS is
\begin{eqnarray*}
(p'^{\downarrow D} \oplus p''^{\downarrow D})(\mathbf{x})&=&p'^{\downarrow D}(\mathbf{x}) + p''^{\downarrow D}(\mathbf{x}))\\
=\sum_{\mathbf{z}^{\downarrow D} = \mathbf{x}} p'(\mathbf{z}) +
\sum_{\mathbf{z}^{\downarrow D} = \mathbf{x}} p''(\mathbf{z})&=&
\sum_{\mathbf{z}^{\downarrow D} = \mathbf{x}} (p' \oplus p'')(\mathbf{z})
\end{eqnarray*}
where we use the associativity and commutativity of $+$.
\end{proof}

As stated earlier, several common instances of valuation algebra can
be considered as semiring induced. We present a couple of important examples below:

\begin{exmp}
\emph{Probability potentials} are semiring induced valuation algebras on $\mathbb{R}^+$ with the semiring operations being the arithmetic addition and multiplication. Also known as \emph{arithmetic potentials}, these describe (unnormalised) probability distributions, and thus inference in probabilistic graphical models.
\end{exmp}
\begin{exmp}
\label{dnf-val}
\emph{Disjunctive normal forms} (abbreviated as DNF) are of the form $\alpha_1 \vee \alpha_2 \cdots \vee \alpha_n$ where $\alpha_i$ is of the form $x_1 \wedge x_2 \wedge \cdots \wedge x_k$ and $x_j$ is a literal; either a logical variable or its negation. All frames are binary reflecting true and false values respectively. DNF potentials are induced by the semiring with + and $\times$ being defined as
$a + b = \max(a, b)$ and $a \times b = \min(a, b)$; which are equivalent to
the definition of logical-or and logical-and.
\end{exmp}

There are many other examples of semiring induced valuation algebras, a
detailed introduction to which can be found in \cite{kohlas2008semiring}. In certain cases,
the valuation algebra induced by the semiring has the idempotent property, i.e. $\phi \otimes \phi = \phi$; then we may use more efficient architectures for local computation such as the Lauritzen-Spiegelhalter architecture \cite{kohlas2003information}.

It is also pertinent to mention that for DNF potentials, one can alternately consider the valuation algebra over the formulae itself instead of the models \cite{kohlas1999propositional}, which simplifies computation extensively. This alternative representation is also an anytime ordered valuation algebra, but we have omitted the proof for the purposes of brevity.

\subsection{Belief functions}

Belief potentials are a generalisation of probability potentials to subsets of the
configuration space in Dempster-Shafer's theory of evidence \cite{shafer1976mathematical}. The advantage of belief potentials over standard probability theory is in their
ability to express partially available information in a manner not possible
in probability theory. This is the reason for the usage of belief functions in sensor network literature, which involves fusion of information from various sources \cite{murphy1998dempster,denoeux2000neural,sentz2002combination,yu2005alert}.

For the instance of belief functions, with the composition operator defined as
$[\phi \oplus \phi']_m(A) = [\phi]_m(A) + [\phi']_m(A)$, where $[\phi]_m$ is the mass function associated with the belief function $\phi$, our framework specialises to anytime inference in belief potentials as described in \cite{Haenni2002103}.

\begin{thm}
Belief functions, along with the composition operator defined above, and the truncation operation
$\rho(\phi,k)$ as the potential that contains the $k$ focal sets
of $\phi$ with the highest masses, form an anytime ordered valuation algebra.
\end{thm}

\begin{proof}
Belief functions already form an ordered valuation algebra \cite{Haenni20041}, as well as permit
anytime inference \cite{Haenni2002103}. The anytime inference algorithm in \cite{Haenni2002103} turns
out to be a specific case of the generic anytime inference framework presented in this article. In particular
if we denote $\oplus := +$ in their notation, and the truncation function $\rho(\phi,k) := \rho_k(\phi)$ then
\cite[Theorem 9,10]{Haenni2002103} shows that belief functions also form an anytime ordered valuation
algebra according to the axioms in Section 3.
\end{proof}

\section{Complexity Analysis}

The anytime inference algorithm presented in Section 3 hides the time complexity of approximate inference by restricting
the accuracy of the valuations. While we don't have an explicit control over the accuracy, we can improve it by allocating more time to the refinement algorithm. In this section, we take an alternative approach of focussing on accuracy and estimating the
time complexity, which also allows us to use a tuning parameter which scales from zero accuracy (null valuations) to the valuation obtained after exact inference.

Since complexity of exact (and approximate) inference depends upon the complexity of the combination operation (usually
the more time-consuming operation among combination and focussing), we consider the specific instance of semiring-induced valuation algebras. As there are $n$ valuations, $\phi_1, \ldots, \phi_n$, the resulting BJT $N$ will have $2n-1$ nodes, $n$ of which are the valuations themselves at the leaves of the tree. We denote the maximum frame size of a variable in the semiring induced valuation algebra as $m := \max \{ |\Omega_x|, x \in V \}$. As we are representing semiring induced valuations in memory in terms of a tuple of
the configuration and its associated value, the number of words required to represent
the configuration is a key component in the time and communication complexity. The upper bound on the size of the configuration space for a valuation is thus $m^{|d(\phi)|}$.

\begin{defn}
The \emph{approximation parameter} $k$ is a tunable parameter that goes from 0 to $m^\omega$, where $\omega$ is the treewidth of the binary join tree $N$. 
\end{defn}
$m^\omega$ is the maximum
size of the configuration space that we have to process during the inward or outward propagation phase of the Shenoy-Shafer algorithm. Now we can define the following.

\begin{defn}
The approximate combination operation $\otimes^k: \Phi \times \Phi \rightarrow \Phi$ is defined as combining
the elements of the configuration space of the valuations in a semiring-induced valuation algebra, until we get $k$ resultant elements.
\end{defn}
\begin{lem}
The complexity of the approximate combination operator $\otimes^k$ is $O(k)$.
\end{lem}
\begin{proof}
The worst-case scenario is when the configuration spaces are independent (no variables in common). Then there is no requirement for common support and we can take
the pairwise multiplication of the elements of the configuration space, till we get $k$ elements, giving us $O(k)$ complexity.
\end{proof}

The \textsc{inward-approx}$(N,k)$ algorithm is defined similarly to the \textsc{inward} algorithm, with the
instances of the time-bound combination operator $\otimes_t$ replaced by the approximate combination operator
$\otimes^k$. In the following, $K(\phi,\psi,k)$ returns $(k_1, k_2)$ such that $\rho(\phi, k_1) \otimes \rho(\psi, k_2)$
has at most $k$ elements.

\hrulefill
\begin{algorithmic}
\Function{inward-approx}{$N,k$}
\State for all $n \in \mathit{leaves}(N)$ do $\phi_s(n) \leftarrow \phi(n)^{-\Delta(n)}$
\While {$\mathit{next}(N) \ne \emptyset$}
\State select $n$ from $\mathit{next}(N)$
\State $(k_1, k_2) \gets K(\phi_s(L(n)), \phi_s(R(n)), k)$
\State $\phi(n) \leftarrow \phi_s(L(n)) \otimes^k \phi_s(R(n))$;
\State $\phi_s(n) \leftarrow \phi(n)^{-\Delta(n)}$
\State $\tau(L(n)) \gets \rho(\phi_s(L(n)), k_1)$
\State $\tau(R(n)) \gets \rho(\phi_s(R(n)), k_2)$
\State $\bar\tau(L(n)) \gets \bar\rho(\phi_s(L(n)), k_1)$
\State $\bar\tau(R(n)) \gets \bar\rho(\phi_s(R(n)), k_2)$
\State $\phi_s(n) \leftarrow \phi(n)^{-\Delta(n)}$
\State $s \leftarrow s-1$

\EndWhile
\EndFunction
\end{algorithmic}
\hrulefill

\begin{thm}
The time complexity of \textsc{inward-approx}$(N,k)$ in the Shenoy-Shafer architecture, with the approximation parameter of $k$, given that there
are $n$ valuations in the knowledgebase is $O((n-1)k)$. 
\end{thm}
\begin{proof}
There are $n-1$ combinations as the number of combinations in the binary join tree is the same as the number of non-leaf nodes. As each combination has
a complexity of $O(k)$, we get a complexity of $O((n-1)k)$. Projection has  a complexity of $O(k)$ as there are $k$ elements in the configuration space, so at most $k-1$ summations, which is the case when we are
marginalising to the null set (equivalent to eliminating all the variables), thus it does not change the asymptotic complexity.
\end{proof}

We get the same time complexity for an analogous $\textsc{refine-approx}$ algorithm, with a modification to lines 6--7 of
$\textsc{refine}$ to combine at most $k$ elements.

\textbf{Matching in the exact inference case}. In the exact inference case, the complexity is known to be in the class \#P-hard. In the discussion on complexity \cite{generic-inference}, Kohlas and Pouly derive the estimate $O(|V|. f(\omega))$ where $\omega$ is the treewidth, with $f(x) = m^x$ for the case of semiring induced valuation algebras with variables having a upper bound frame size of $m$. $|V|$ is the number of vertices in the join tree. Substituting $|V| = n, k = m^\omega$ in the time complexity $O(n-1)k$ and taking $k = m^\omega$, we get the same time complexity as the exact inference case; thus the approximate time complexity obtained in terms of the approximation parameter $k$ gives us a transition from $k=0$ (null valuations, obtained when we set the $t=0$ in $\textsc{inward}(N,t)$) to $k=m^\omega$, the exact inference case.

\textbf{Estimation of accuracy from elapsed time}.
It can be useful to derive an estimate of the accuracy of a valuation given the elapsed time of the algorithm in specific cases. Here, we shall consider the example of probability potentials. The time-bound
combination operator combines the configurations with the largest weight first so that we get diminishing returns; the accuracy also depends on the sparsity of the probability potential. For simplicity we consider uniform distributions, where the weights are uniformly distributed in the configuration space. Then we can state the following:
\begin{lem}
The fractional error estimate compared to the exact probability potential is
\begin{equation}
\epsilon(t) = 1 - \max{\left(1,\frac{t}{m^\omega c(n-1)}\right)}
\end{equation}
where $\omega$ is the treewidth, $c$ is the constant time required to combine two elements in the configuration space, and $n$ is the number of valuations in the knowledgebase.
\end{lem}
\begin{proof}
As each configuration has an uniform weight, the accuracy of combination at the root node (which is the solution to the inference problem obtained from the inward propagation algorithm) is directly proportional to the allocated time which is on average $t/(n-1)$ as there are $n-1$ combinations. Considering that each combination takes $c$ units, and in the worst-case each configuration has weight $1/m^\omega$ (for a normalised potential; for unnormalised, this introduces a constant factor which is cancelled out by considering a fractional error estimate), we get the fractional error estimate as above.
\end{proof}
As can be easily seen, $\epsilon(0) = 1$, and $\epsilon(O((n-1).m^\omega)) = 0$ where $O((n-1).m^\omega)$ is the exact inference time complexity.

\section{Implementation}
We implemented the anytime inference algorithm using the Python programming language, on a Core i5 CPU with 4GB RAM. While we have shown anytime inference in a Bayesian network here, the framework, being generic, can be applied to other valuation algebras which satisfy the necessary axioms.

\begin{figure}[htp]
\centering
\includegraphics[scale=0.6]{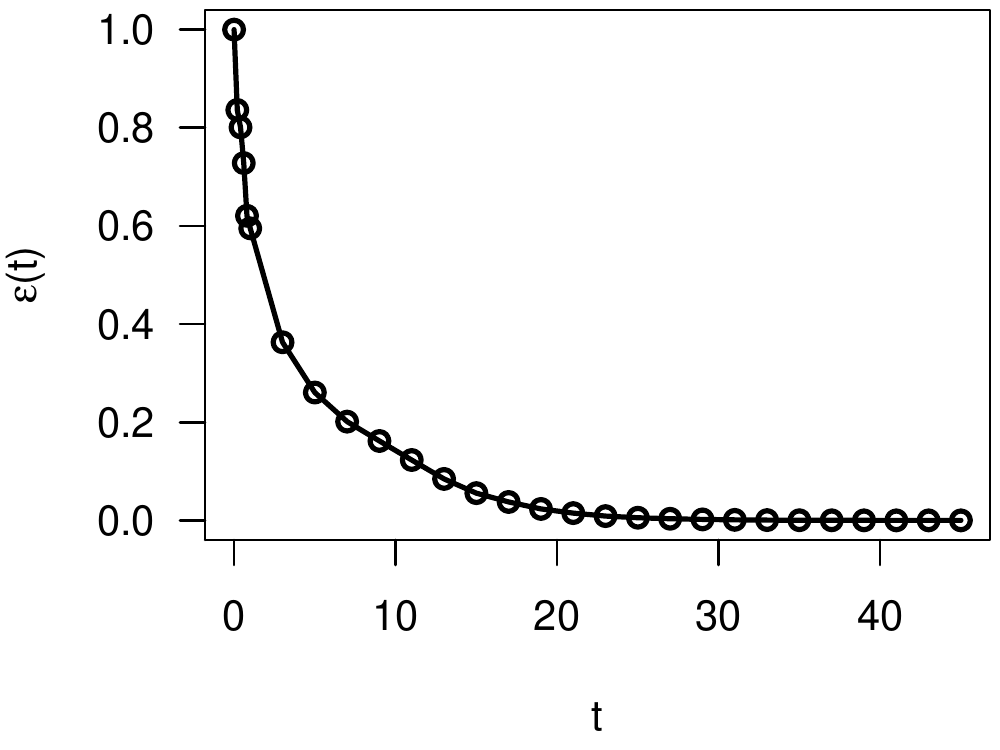}
\caption{Anytime inference progress in the \textsc{child} dataset}
\end{figure}

The figure shows progress of anytime inference on the \textsc{child} dataset, which was used as a
case study for exact inference in \cite{cowell2007probabilistic}. The progress is shown as a function of the fractional error estimate with time units (the actual total time for the series of successive refinements, up to the exact valuation is $<10\text{s}$):
\begin{equation}
\epsilon(t) = 1 - \frac{\sum L_{\phi_t}}{\sum L_{\phi}}
\end{equation}
Here the sum is over the weights of the configurations $L_\phi$ of a valuation $\phi$; $\phi_t$ is the valuation obtained at the root
after time $t$, and $\phi$ is the exact valuation. As expected, the fractional error estimate
converges to zero as we obtain the exact valuation.

\section{Conclusion}
In this work, we have shown that we can construct anytime algorithms for
generic classes of valuation algebras, provided certain conditions are satisfied. We have also shown that the important subclass of semiring induced valuation algebras admit an anytime inference algorithm as they meet the aforementioned conditions. This is useful as semiring induced valuation algebras include
important valuation algebra instances like probability potentials,
DNF potentials and relational algebras, among others.

From a broader perspective, the advantage of operating in the generic framework of valuation algebras has been addressed before \cite{generic-inference}; we can target a large class of problems using a unified framework; the inference or projection problem can be found in various forms: Fourier transforms, linear programming and constraint satisfaction problems. Enriching the valuation algebra structure through extensions is thus useful. Anytime inference in particular has a wide spectrum of applications. We also plan to study the applicability of our framework across these various domains in future work.

We are currently working on implementation of other instances of anytime ordered valuation algebras, as well as conducting a complexity analysis of the algorithm in a distributed setting using the Bulk Synchronous Parallel \cite{bsp} model.
\bibliographystyle{ecai}
\pagebreak
\bibliography{ecai}
\end{document}